\documentclass{article}

\usepackage[utf8]{inputenc}
\usepackage[english]{babel}

\usepackage{amsmath,amssymb,amsthm}
\usepackage{hyperref}
\usepackage{tikz}
\usepackage{enumitem}
\usepackage{multicol}
\usepackage{mathtools, nicefrac}
\usepackage{booktabs}
\usepackage{graphicx, caption,subcaption}
\usepackage{float}
\usepackage{authblk}
\usepackage{nicefrac}
\usepackage[vlined,ruled]{algorithm2e}
\usepackage{etoolbox}
\usepackage[left=3cm, right=3cm, top=3.5cm,bottom=3.5cm]{geometry}

\usepackage[backend=biber,  bibstyle=ieee, citestyle=numeric-comp]{biblatex}
\providecommand{\keywords}[1]
{
	\small	
	\textbf{\textit{Keywords---}} #1
}
\hypersetup{colorlinks, citecolor=red}

\newtheorem{definition}{Definition}
\AfterEndEnvironment{definition}{\noindent\ignorespaces}
\newtheorem{theorem}{Theorem}
\theoremstyle{definition}
\newtheorem{example}{Example}
\newtheorem{property}{Property}
\newtheorem{corollary}{Corollary}

\newcommand*\samethanks[1][\value{footnote}]{\footnotemark[#1]}

\newcommand\fract[5]{\ensuremath{\prescript{#1}{#3}{\mathcal{#5}}_{#4}^{#2}}}

\begin{document}
	\flushbottom
	
 \title{Accelerating Fractional PINNs using Operational Matrices of Derivative}

\author[1]{Tayebeh Taheri \thanks{Email: \{ttaherii1401, alirezaafzalaghaei\}@gmail.com}}
 
\author[1]{Alireza Afzal Aghaei\samethanks}	
 \author[1,2,3]{Kourosh Parand \thanks{Email: k\_parand@sbu.ac.ir, Corresponding author}}

	\affil[1]{\small{Department of Computer and Data Sciences, Faculty of Mathematical Sciences, Shahid Beheshti University, G.C. Tehran, Iran}}
	\affil[2]{\small{Department of Cognitive Modeling, Institute for Cognitive and Brain Sciences, Shahid Beheshti University, G.C. Tehran, Iran}}
	\affil[3]{\small{Department of Statistics and Actuarial Science, University of Waterloo, Waterloo, Canada}}

	\maketitle

	\begin{abstract}
This paper presents a novel operational matrix method to accelerate the training of fractional Physics-Informed Neural Networks (fPINNs). Our approach involves a non-uniform discretization of the fractional Caputo operator, facilitating swift computation of fractional derivatives within Caputo-type fractional differential problems with $0<\alpha<1$. In this methodology, the operational matrix is precomputed, and during the training phase, automatic differentiation is replaced with a matrix-vector product. While our methodology is compatible with any network, we particularly highlight its successful implementation in PINNs, emphasizing the enhanced accuracy achieved when utilizing the Legendre Neural Block (LNB) architecture. LNB incorporates Legendre polynomials into the PINN structure, providing a significant boost in accuracy. The effectiveness of our proposed method is validated across diverse differential equations, including Delay Differential Equations (DDEs) and Systems of Differential Algebraic Equations (DAEs). To demonstrate its versatility, we extend the application of the method to systems of differential equations, specifically addressing nonlinear Pantograph fractional-order DDEs/DAEs. The results are supported by a comprehensive analysis of numerical outcomes.

	\end{abstract}
	
	\keywords{Physics-informed neural networks, nonlinear differential equations, Fractional derivative, Operational matrix}

 	\section{Introduction}
Fractional calculus (FC) \cite{west2016fractional}, a specialized field in mathematical analysis, expands traditional differentiation and integration by embracing non-integer orders.  Fractional calculus has gained substantial prominence over four decades. As an evolving arena, FC anticipates the introduction of numerous models for real-world applications in science and engineering, particularly in areas where nonlocality plays a crucial role. This burgeoning field demonstrates profound applications across diverse scientific disciplines, extending its reach into the dynamics of the complex real world, with new ideas implemented and tested using real data. Here, some of these applications are mentioned. In the domain of underwater sediment and biomedical applications, fractional derivative models prove valuable for better understanding wave propagation, providing insights into absorption mechanisms grounded in the relaxation processes observed in materials such as polymers \cite{pandey2016linking}. In \cite{zeng2014optimal}, a Continuous Time Random Walk (CTRW) optimal search framework is introduced for specific targets with unknown locations. This approach employs fractional calculus techniques to ascertain optimal distributions for both search length and waiting time. 

Fractional order systems, in contrast to memoryless integer-order dynamic systems, feature long memory and the application of fractional calculus improves and extends established control methods and strategies \cite{monje2010fractional}. Its applications in diverse scientific and engineering fields, including image processing \cite{zhang2012adaptive}, underscore its recognition, providing valuable tools for solving differential equations,  integral differential equations, partial differential equations (PDEs), differential-algebraic equations (DAEs), and delay differential equations (DDEs).  The utilization of discrete fractional calculus emerges as a practical tool to address engineering challenges in discrete time or space structures, particularly those entailing delays inherent in delay differential equations.  This approach provides stability theory for fractional difference equations, facilitating long-term control and mitigating potential errors linked to numerical discretization in continuous fractional calculus \cite{wu2018finite}.

DDEs represent a distinct class of differential equations where the unknown functions depend on previous time states, setting them apart from conventional Ordinary Differential Equations (ODEs). DDEs have diverse applications spanning mathematics, biological systems, engineering, physics, economics and finance, chemical reactions, ecology, communication networks, weather and climate, and medicine \cite{smith2011introduction, driver2012ordinary, rihan2021delay}. For example, in biological systems, DDEs are used to model the interactions and time delays in predator-prey systems, disease spread, and ecological systems. Also, DDEs are employed to describe neural oscillations and synchronization in the brain, considering the time delays in signal transmission between neurons. In many of these applications, DDEs provide a more accurate representation of real-world phenomena compared to ODEs because they can capture the impact of time delays on system dynamics. In addition to standard delay differential equations, there are several other types of delay equations used in engineering to model systems with time delays. Pantograph delay differential equations (PDDEs) introduce a more intricate structure by incorporating both forward and backward time delays in the equation. PDDEs have arisen in the modeling of pantograph control systems, cell growth models \cite{van2011mellin}, population growth \cite{aiello1992analysis}, and electric locomotive \cite{ockendon1971dynamics}. Fractional differential equations \cite{dzielinski2022fractional} give significant meanings in physical contexts,  hence empowering researchers to capture more realistic and nuanced behavior, as demonstrated in vibration theory \cite{huseynov2021class}, and biological systems with memory \cite{barros2021memory}.

  Dealing with delay differential equations can be intricate due to the time delays involved. The primary methods employed to solve these equations include numerical, analytical, and semi-analytical approaches.
   One of the most widely recognized analytical techniques for solving DDEs is the method of steps \cite{bauer2013solving}. This involves breaking the DDE into small time intervals and solving it as a sequence of ODEs at discrete time points. Laplace transform techniques, as demonstrated in the work by \cite{yang2023stability}, is another analytical method. Perturbation-iteration methods, as discussed in the study by Bahsi \cite{bahcsi2015numerical}, fall under the category of semi-analytical approaches. The analytical and semi-analytical technique enhances precision and stability while potentially incurring higher computational costs.
   
   To address the challenges posed by complex systems featuring time delays, numerical methods for solving delay differential equations play a crucial role. A variety of numerical techniques have been developed to efficiently handle DDEs. Finite difference methods \cite{johnson2023investigation} constitute a category of numerical techniques utilized for the solution of differential equations, involving the approximation of derivatives through finite differences. Examples of finite difference methods encompass Runge-Kutta techniques \cite{senu2022numerical}, Euler method \cite{ ngoc2022stability}, and Linear multistep approaches, as presented in Shaalini's work \cite{shaalini2021new}.

   Spectral methods belong to a group of numerical techniques that make use of spectral (or Fourier) representations of the solution. They prove to be particularly effective when addressing DDEs that display periodic or oscillatory characteristics. Illustrative instances of this approach for solving DDEs involve collocation techniques \cite{sriwastav2023numerical}, methods utilizing wavelets in conjunction with collocation \cite{faheem2022wavelet}, the application of the Bernoulli collocation method \cite{adel2020solving}, and the utilization of the transferred Legendre pseudospectral method \cite{jafari2021new}. The performance of spectral methods is profoundly influenced by the selection of basis functions, including Fourier, Chebyshev, or Legendre, as well as the initial values, introducing challenges in the decision-making process. Finite element methods, such as the Chebyshev Tau method \cite{abd2022hypergeometric} and the Chebyshev discretization scheme \cite{liu2019stability}, along with meshless methods like radial basis functions (RBFs) \cite{ABBASZADEH202144} and meshless local Petrov-Galerkin techniques \cite{TAKHTABNOOS201767}, represent promising numerical approaches for solving DDEs. Meshless methods often rely on local approximations, such as radial basis functions or moving least squares, leading to difficulties in capturing global features of the solution. This limitation can result in reduced accuracy, particularly when dealing with DDEs that exhibit complex and non-local dynamics.

   While DDEs capture systems with temporal dependencies, differential-algebraic equations (DAEs) are a class of differential equations that describe dynamic systems with algebraic constraints, and both play a crucial role in modeling complex phenomena. Their applications are evident in a wide spectrum of disciplines, from aerospace engineering and biomechanics to geophysics and robotics. To effectively address the diverse applications of DAEs across various domains, it is essential to employ specialized numerical methods for their solution. Instances of these techniques comprise a variety of methods, including the Adomian Decomposition Method (ADM) \cite{celik2006solution}, Pade series \cite{ccelik2003numerical}, homotopy perturbation \cite{soltanian2010solution}, and homotopy analysis technique \cite{zurigat2010analytical}, among others, all contributing to the effective solution of DAEs.  Common issues and limitations associated with numerical methods for DDEs and DAEs include concerns related to stability, convergence, accuracy, and computational demands.
   
Machine learning is gaining traction in solving differential equations, offering flexibility and efficacy compared to traditional methods. Its adaptability and ability to learn complex patterns result in more accurate solutions and faster predictions with reduced computational costs, exhibiting robust generalization across various data distributions and reduced sensitivity to discretization choices. In specific contexts, such as solving fractional-order differential equations, it has exhibited superior performance through the utilization of techniques like least-squares support vector regression \cite{taheri2023bridging}. In applying machine learning to solve differential equations, a crucial aspect lies in the specialized use of artificial neural networks (ANN), showcasing their effectiveness across a diverse range of differential equation types. Highlighting instances of artificial neural network applications in solving differential equations, Rahimkhani and Ordokhani \cite{rahimkhani2021orthonormal} introduced the orthonormal Bernoulli wavelets neural network methodology to address the Lane–Emden equation. Sabir et al. \cite{sabir2022intelligent} applied an optimization approach for ANNs using a blend of genetic algorithms and sequential quadratic programming to resolve multi-pantograph delay differential equations. Khan et al. \cite{khan2022design} devised intelligent networks with backpropagation, while Panghal et al. \cite{panghal2022neural} implemented feed-forward ANNs for solutions in delay differential systems. Ye et al. \cite{ye2023slenn} innovated with the shifted Legendre neural network method, employing an extreme learning machine algorithm (SLeNN-ELM) for solving fractional differential equations with constant and proportional delays. Zhang et al. \cite{zhang2023sync} suggested an approach utilizing neural networks to solve delay differential equations and derive control policies from specific regions of controlled systems. Ruthotto et al. \cite{ruthotto2020deep} adopted a deep convolutional neural network to solve partial differential equations.

Incorporating orthogonal functions as activation functions in neural networks enhances their effectiveness in solving differential equations. Methods like the orthonormal Bernoulli wavelets neural network and the shifted Legendre neural network illustrate the synergy between spectral activation functions and the networks' capacity to capture complex mathematical relationships.
In diverse applications, orthogonal polynomials like Legendre, Chebyshev, and Jacobi are effectively used as activation functions in neural networks. They contribute to the accuracy of solving various integral and partial differential equations, showcasing the versatility of employing specialized functions in computational methodologies \cite{hajimohammadi2021fractional, parand2023neural}.

Recently, a deep neural network framework known as Physics-Informed Neural Networks (PINNs) has been introduced by Raissi et al. \cite{raissi2019physics}. The PINNs methodology combines the power of neural networks with the principles of physics, offering a versatile and effective tool for addressing a wide range of differential equations, including applications to solving stationary PDEs \cite{peng2022rpinns}, nonlinear integro-differential equations \cite{yuan2022pinn},  heat transfer problems \cite{cai2021physics}, and dispersive PDEs \cite{bai2021physics}. Furthermore, PINNs have demonstrated success in solving fractional differential equations, addressing challenges in areas such as integro-differential equations of fractional order \cite{pang2019fpinns}, fractional differential equations revealing time-dependent parameters and data-driven dynamics \cite{kharazmi2021identifiability}, and Hausdorff derivative Poisson equations \cite{wu2023physics}. In these studies, the researchers employ discretization methods for fractional derivatives. Unlike integer calculus, automatic differentiation cannot be applied to fractional calculus due to the inapplicability of standard rules in this context.

In this paper, we introduce an alternative non-uniform discretization approach for Caputo-based fractional differential equations, specifically tailored for integration into neural networks. Notably, we achieve the efficient computation of the Caputo fractional derivative ($\alpha$) through the utilization of operational matrices. Then, by utilizing a special extension of PINNs, namely Legendre Deep Neural Network (LDNN) our method adeptly tackles the complexities inherent in fractional differential equations. Overall,  the core aim of this article is to:
  \begin{itemize}
  \item{ Proposed a physics-informed neural network for solving fractional-order differential equations.}
  \item{ Proposed non-uniform discretization approach for Caputo-based FDEs.}
  \item{Utilize hidden layers within this network, incorporating Legendre and Chebyshev polynomials as activation functions.}
  \item{Develop an efficient Artificial Neural Network for solving:
  \begin{itemize}
   \item{Pantograph delay differential equations.}
   \item{Nonlinear delay differential equations.}
   \item{Fractional delay differential equations.}
   \item{Differential-algebraic equations.}
    \item{Linear fractional differential-algebraic equations.}
  \end{itemize}
  }

  \end{itemize}
  
The rest of this paper is organized as follows: section \ref{sec2} outlines techniques for calculating both the fractional derivative and the operational matrix of the Caputo fractional derivative. In the methodology section \ref{sec3}, we clarified the proposed neural network method and its structure. To facilitate a comprehensive understanding of the procedure, we introduce our network to specifically address the resolution of differential equations with fractional orders. We have assessed the efficacy of our method by employing it to approximate various problems in the numerical examples section \ref{sec4}. Finally, an overview of the proposed method is presented in the conclusion section \ref{sec5}.

 	\section{Fractional Calculus} \label{sec2}
 	In this section, we present methods for computing the fractional derivative and the operation matrix of the Caputo fractional derivative.

 	\begin{definition}
 	  Let $f$ be a continuous function on the interval $[a,b]$, the Reimann-Liouville fractional derivative of arbitrary order $\alpha$ is defined as:
 	  \begin{equation}
 	       \fract{}{\alpha}{}{}{D}f(x)=\frac{1}{\Gamma(k-\beta)}\frac{d^k}{dx^k}\int_{a}^{x}(x-t)^{k-\alpha-1}f(t)dt, \quad k-1\leq \alpha \leq k,
 	  \end{equation}
 	where $\Gamma(.)$ is the Gamma function denoted by the following integral:
 	    \begin{equation}
 	  \Gamma(z)=\int_{0}^{\infty} e^{-t} t^{z-1}dt,
    \end{equation}
 which is convergent on the right half of the complex plane $(Re(z)\geq 0)$.
 	  \end{definition}

 	   \begin{definition}
 	     Caputo's definition of the fractional order of $\alpha$  is expressed as follows:
 	     \begin{equation}
 	           \fract{C}{\alpha}{}{}{D}f(x)=\frac{1}{\Gamma(n-\alpha)}\int_{a}^{x} \frac{f^{(n)}(t)dt}{(x-t)^{\alpha+1-n}}, \quad n-1<\alpha<n.
 	     \end{equation}
 	   \end{definition}

\begin{property} \label{property}
  Let's denote the fractional order as $q$, where $q=n+\alpha$ and $n$  is the integer part of $q$.  The fractional derivative of order $q$ can be expressed as:
  \[ D^q[f(t)] = D^n[D^{1-\alpha}[f(t)]], \]
here, \(D^n\) represents the derivative of integer order \(n\), and \(D^{1-\alpha}\) represents the fractional derivative of order \(1-\alpha\).

\end{property}
\begin{theorem}
    Let $0 < \alpha < 1$ and the interval $[t_0, t_n]$ is discretized to $n+1$ points, $0 = t_0 < t_1 < \dots < t_n$. Then the following linear combination approximates the Caputo fractional derivative of order $\alpha$:
    \begin{equation}
        {}_{0}^{C}\mathcal{D}_{t_n}^{\alpha}f{(t)} = \sum_{k=0}^n\omega_k f(t_k),
    \end{equation}
    where $f$ is the desired function and $\omega_k$ are real-valued weights.
    \label{thm:L1}
\end{theorem}
\begin{proof}
We start by splitting the integration into $n$ non-equidistant intervals
    \begin{equation*}
        \begin{aligned}
{}_{0}^{C}\mathcal{D}_{t_n}^{\alpha}f{(t)} &=\frac{1}{\Gamma{(1-\alpha)}}\int_{0}^{t_{n}}\frac{f^{\prime}(x)}{(t_{n}-x)^{\alpha}}dx  \\
&=\frac1{\Gamma(1-\alpha)}\sum_{k=0}^{n-1}\int_{t_k}^{t_{k+1}}\frac1{(t_n-x)^\alpha}f^{\prime}(x)dx \\
&\approx\frac1{\Gamma(1-\alpha)}\sum_{k=0}^{n-1}\int_{t_k}^{t_{k+1}}\frac{1}{(t_n-x)^k}\frac{f(t_{k+1})-f(t_k)}{t_{k+1}-t_k}dx. \\
\end{aligned}
    \end{equation*}
    
In each interval one may approximate the $f'(x)$ using the forward scheme finite difference approximation of the derivative:
    \begin{equation*}
\begin{aligned}
{}_{0}^{C}\mathcal{D}_{t_n}^{\alpha}f{(t)} & \approx\frac1{\Gamma(1-\alpha)}\sum_{k=0}^{n-1}\frac{f(t_{k+1})-f(t_k)}{t_{k+1}-t_k}\int_{t_k}^{t_{k+1}}\frac{dx}{(t_n-x)^\alpha} .\\
\end{aligned}
    \end{equation*}

 Then, the analytical integration can be utilized to simplify the formulation:
    \begin{equation*}
        \begin{aligned}
{}_{0}^{C}\mathcal{D}_{t_n}^{\alpha}f{(t)} & \approx \frac1{\Gamma(1-\alpha)}\sum_{k=0}^{n-1}\frac{f(t_{k+1})-f(t_{k+1})}{t_{k+1}-t_{k}}\cdot\left[-\frac{(t_{n}-t_{k})^{1-\alpha}-(t_{n}-t_{k+1})^{1-\alpha}}{(\alpha-1)}\right] \\
&\approx\frac1{\Gamma(2-\alpha)}\sum_{k=0}^{n-1}\left[\frac{(t_n-t_k)^{1-\alpha}-t_{n-k+1})^{1-\alpha}}{t_{k+1}-t_k}\right]\left[f(t_{k})-f(t_{k+1})\right].
        \end{aligned}
    \end{equation*}
 
 Defining $\mu_k$ as the weight part of the summation, and a simple reformulation of it yields:
\begin{equation*}
    \begin{aligned}
        {}_{0}^{C}\mathcal{D}_{t_n}^{\alpha}f{(t)} &\approx\frac{1}{\Gamma{(2-\alpha)}}\sum_{k=0}^{n-1}\mu_{k}\left[f{(t_k)}-f{(t_{k+1})}\right] \\
&\approx\frac{1}{\Gamma(2-\alpha)}\sum_{k=0}^{n-1} \left[ \mu_{k}-\mu_{k-1}\right] f(t_{k}) \\
&\approx\sum_{k=0}^n\omega_k f(t_k), 
    \end{aligned}
\end{equation*}
where $\omega_k = \nicefrac{\left[\mu_k - \mu_{k-1}\right]}{\Gamma(2-\alpha)}$ and:
\begin{equation*}
    \mu_k = \begin{cases}\displaystyle\frac{(t_n-t_k)^{1-\alpha}-t_{n-k+1})^{1-\alpha}}{t_{k+1}-t_k} & 0\le k < n\\ 0 & otherwise.\end{cases}
\end{equation*}
\end{proof}
\begin{corollary}
    for the equidistant data points $t_k = kh$ for step size $h$, the formulation of $\mu_k$ can be simplified as:
    \begin{equation*}
        \mu_k = \begin{cases}\displaystyle\frac{((n-k)h)^{1-\alpha}-{(n-k-1)h})^{1-\alpha}}{h} & 0\le k < n\\ 0 & otherwise.\end{cases}
    \end{equation*}
\end{corollary}
\begin{proof}
    The proof is straightforward.
\end{proof}
\begin{corollary}
    For computing the fractional derivative of Caputo type for $\alpha>1$, one can calculate the derivative of integer order, followed by the computation of the derivative of order $1-\alpha$ where $\alpha$ represents the fractional part of the specified order. (see property \ref{property})
\end{corollary}
\begin{theorem}
    The operation matrix of the Caputo fractional derivative can be obtained using the lower triangular matrix $\mathcal{A}$:
    \begin{equation*}\label{equ: matrix derivation}
        \mathcal{A} = \begin{bmatrix}
0 \\
\omega_0^{(1)} & \omega_1^{(1)}\\
\omega_0^{(2)} & \omega_1^{(2)} & \omega_2^{(2)}\\
&\vdots\\
\omega_0^{(n-2)} & \omega_1^{(n-2)} & \omega_2^{(n-2)} &\omega_3^{(n-2)} & \cdots &\omega_{n-2}^{(n-2)}\\
%&&&&&\ddots\\
\omega_0^{(n-1)} & \omega_1^{(n-1)} & \omega_2^{(n-1)} &\omega_3^{(n-1)} & \cdots &\omega_{n-2}^{(n-1)} & \omega_{n-1}^{(n-1)}
\end{bmatrix}
    \end{equation*}
    
Hence, for the vector-valued function $\mathbf{f}_i = f(x_i)$ and arbitrary nodes $x_i$ the Caputo fractional derivative of order $\alpha$ can be efficiently computed using the operational matrix of derivative:
    \begin{equation*}
        _{}^{C}\mathbf{f}^\alpha \approx \mathcal{A} \mathbf{f}.
    \end{equation*}
\end{theorem}
\begin{proof}
    The proof can be completed if we approximate the $i$-th element of $_{}^{C}\mathbf{f}^\alpha$ using theorem \ref{thm:L1}:
    \begin{equation*}
        _{}^{C}\mathbf{f}^\alpha_i = \sum_{k=0}^i\omega_k^{(i)} \mathbf{f}_i,
    \end{equation*}
where $\omega_k^{(i)} = \nicefrac{\left[\mu_k^{(i)} - \mu_{k-1}^{(i)}\right]}{\Gamma(2-\alpha)}$ and:
\begin{equation*}
    \mu_k^{(i)} = \begin{cases}\displaystyle\frac{(t_i-t_k)^{1-\alpha}-t_{i-k+1})^{1-\alpha}}{t_{k+1}-t_k} & 0\le k < n\\ 0 & otherwise.\end{cases}
\end{equation*}
\end{proof}

\section{Methodology} \label{sec3}

In this section, we introduce neural network architecture designed for solving fractional delay differential equations. Importantly, we highlight the inherent adaptability of this approach, indicating its potential extension to address differential-algebraic equations. First, we consider a general  form of fractional delay differential equation presented below:

 	\begin{equation}
 	   \phi^{(\alpha)}(\tau)=\chi(\tau,\phi(\tau),\phi^{\prime}(\tau),\phi^{\prime\prime}(\tau),\dots,\phi^{(\lfloor\alpha\rfloor)}(\tau),\phi(\delta(\tau))),
 	\end{equation}
 	where $\delta$ represents the non-negative delay term, and $\alpha$ denotes the order of the derivative.
Now, in the vector form the residual function can be defined as follows:
\begin{equation}
    \mathcal{R}(\boldsymbol\tau)= \mathcal A\tilde\phi(\boldsymbol\tau)-\chi(\boldsymbol\tau,\tilde\phi(\boldsymbol\tau),\tilde\phi^{\prime}(\boldsymbol\tau),\tilde\phi^{\prime\prime}(\boldsymbol\tau),\dots,\tilde\phi^{(\lfloor\alpha\rfloor)}(\boldsymbol\tau),\tilde\phi(\delta(\boldsymbol\tau))),
\end{equation}
where $ \tilde{\phi}(\boldsymbol\tau)$ are the predicted values by the neural network for the input batch $\boldsymbol{\tau}$ and $\mathcal{A}$ is the operation matrix to calculate the fractional derivative, that is defined in the equation (\ref{equ: matrix derivation}).

In crafting the loss function, we draw inspiration from the Physics-Informed Neural Network, integrating both the initial values of the differential equations and the residual values. It is crucial to take into consideration the initial conditions of the delay differential equation, as these conditions are inherently guided by the physical significance of the problem. Below, we have explicitly outlined these conditions.

    If  $\phi^{(p)}(a)=k_{p}, p=0,1,...,\lceil \alpha \rceil$ are the initial values of the differential equation, we define the boundary conditions as follows:
    \begin{equation}
        \mathcal{B}_p = \tilde{\phi}^{(p)}(a) - k_{p}.
    \end{equation}
    
    As a result, the loss function is defined as:

    \begin{equation}
        loss(\boldsymbol{\tau}) = \lambda \| \mathcal{R}({\boldsymbol\tau})\|_2 + \sum_{p=0}^{\lceil \alpha \rceil} \mathcal{B}_p^{2},
    \end{equation}
    where $\|.\|$ is the $L_2$ norm. To enhance the effectiveness of the residual function in the loss function, we multiplied its value by the weight $\lambda$.

Extending the proposed neural network method initially devised for solving fractional delay differential equations, we now apply this approach to the general case fractional-order system of differential-algebraic equations, broadening the scope of its applicability.  The general form of a DAE can be expressed as:
\begin{align*}
\chi\left(\tau, \phi(\tau), \phi^{\prime}(\tau), \phi^{\prime\prime}(\tau),\dots,\phi^{(\alpha)}(\tau)\right) &= 0, \\
\zeta\left(\tau, \phi(\tau)\right) &= 0,
\end{align*}
where 
\begin{itemize}
    \item $\tau$ is the independent variable,
    \item $\phi \in \mathbb{R}^{M}$ is the vector of state variables,
    \item $\chi$ is a vector-valued function that describes the dynamic behavior of the system,
    \item $\zeta$ is a vector-valued function representing algebraic constraints.  
\end{itemize}

  Now, the residuals for each $j$ are formulated as follows:

\begin{equation}
\mathcal{R}_j(\boldsymbol\tau) = \chi_j\left(\boldsymbol\tau, \tilde\phi(\boldsymbol\tau), \tilde\phi^{\prime}(\boldsymbol\tau), \tilde\phi^{\prime\prime}(\boldsymbol\tau),\dots,\mathcal{A}\tilde\phi(\boldsymbol\tau)\right) - \zeta_j\left(\boldsymbol\tau, \tilde\phi(\boldsymbol\tau)\right),
\end{equation}
where $m$ is the number of equations in a DAE and $\boldsymbol{\tau}$ is the input batch. Each equation is addressed with a dedicated neural network architecture, akin to the previously introduced models for solving DDEs, ensuring uniformity across the networks. This design integrates activation functions consistent with those employed in our prior frameworks. In addressing the system of DAEs, the loss function is formulated, drawing inspiration from the previously discussed neural network architectures for DDEs.

    If $\phi_j(a)=k_{j}$ are the initial values of a DAE, the boundary conditions for each $j$ are defined as follows:
    \begin{equation}
        \mathcal{B}_j= \tilde{\phi_j}(a) - k_{j}.
    \end{equation}
    
     The loss function is defined as:

    \begin{equation}
        loss(\boldsymbol\tau) = \lambda \frac{1}{m}\sum_{i=1}^{m} \| \mathcal{R}_i(\tau)\|_2 +  \left(\sum_{j=1}^{m} (\mathcal{B}_{j}^{2})\right),
    \end{equation}
    where $n$ is the number of discrete points chosen for solving the equation and $\tilde{\phi_j}(a)$ are the predicted values by the neural network. To amplify the impact of the residual function within the loss function, its value is weighted by the parameter $\lambda$. 
    
   The architecture of the network in the proposed approach can be any neural network type. For example a pure MLP network \cite{raissi2019physics}, an Recurrent Neural Network (RNN) \cite{hagge2017solving} or a Legendre block \cite{parand2023neural} can be employed.  Here we chose the Legendre block which is an extension of the MLP network with better accuracy. 

Fig. \ref{fig: Legendre block} visually depicts the graphical representation of neural block. Formally, this block is defined by 
$$\mathbf y_i = v_i\left(\tanh(\sum_{j=1}^k\mathbf{w}_j \mathbf x_j+b)\right), \quad i=1,2,\dots,n,$$
where  \(v_{n}(x)\) represents \(n\)-th Legendre polynomial, serving as an orthogonal activation function, defined by:
\begin{center}
 	    $v_0(x)=1$, \quad $v_1(x)=x$, 
 	    
 	    $(n+1)v_{n+1}(x)=(2n+1)xv_n(x)-nv_{n-1}(x)$, \quad $n\geq1$.
 	\end{center}

The choice of $tanh(x)$ is due to the domain restriction of Legendre polynomials. This integration is motivated by the unique mathematical properties of Legendre polynomials, synergizing with the nonlinear characteristics of $tanh(x)$.
Crucially, the computational complexity associated with these blocks is relatively low. Leveraging operational matrices during the back-propagation phase facilitates the computation of the gradient of the cost function concerning network weights. This operational matrix is defined by:
\begin{equation}
			\label{legendre first derivative}	
			\frac{\mathrm{d} }{\mathrm{d} x} \mathcal{V} (x)= A\mathcal{V }(x)\;,	
		\end{equation}
  where
  \begin{equation}
			\label{difLegendr}
			A = ( a_{i,j} ) =\begin{cases}
				2j+1 & j=i-k  \\ 
				0& o.w. \; 
			\end{cases} ,
		\end{equation}
		where 
		\begin{equation}
			k=\begin{cases}
 				1,3,...,m& \; if \;m\; is \; Odd \\ 
 				1,3,...,m-1& \; if \;m\; is \; Even \\ 
			\end{cases} .
 		\end{equation}

Consequently, the overall complexity is minimized, with potential intricacies in the calculation of Legendre polynomials being negligible due to the use of a small $n$ in the blocks. These blocks exhibit flexibility, allowing placement anywhere in the network. Their incorporation enables the embedding of polynomial spaces within the neural network, contributing significantly to the network's expressive capacity.

	\begin{figure}[!ht]
		\centering 
		\includegraphics[width=.6\linewidth]{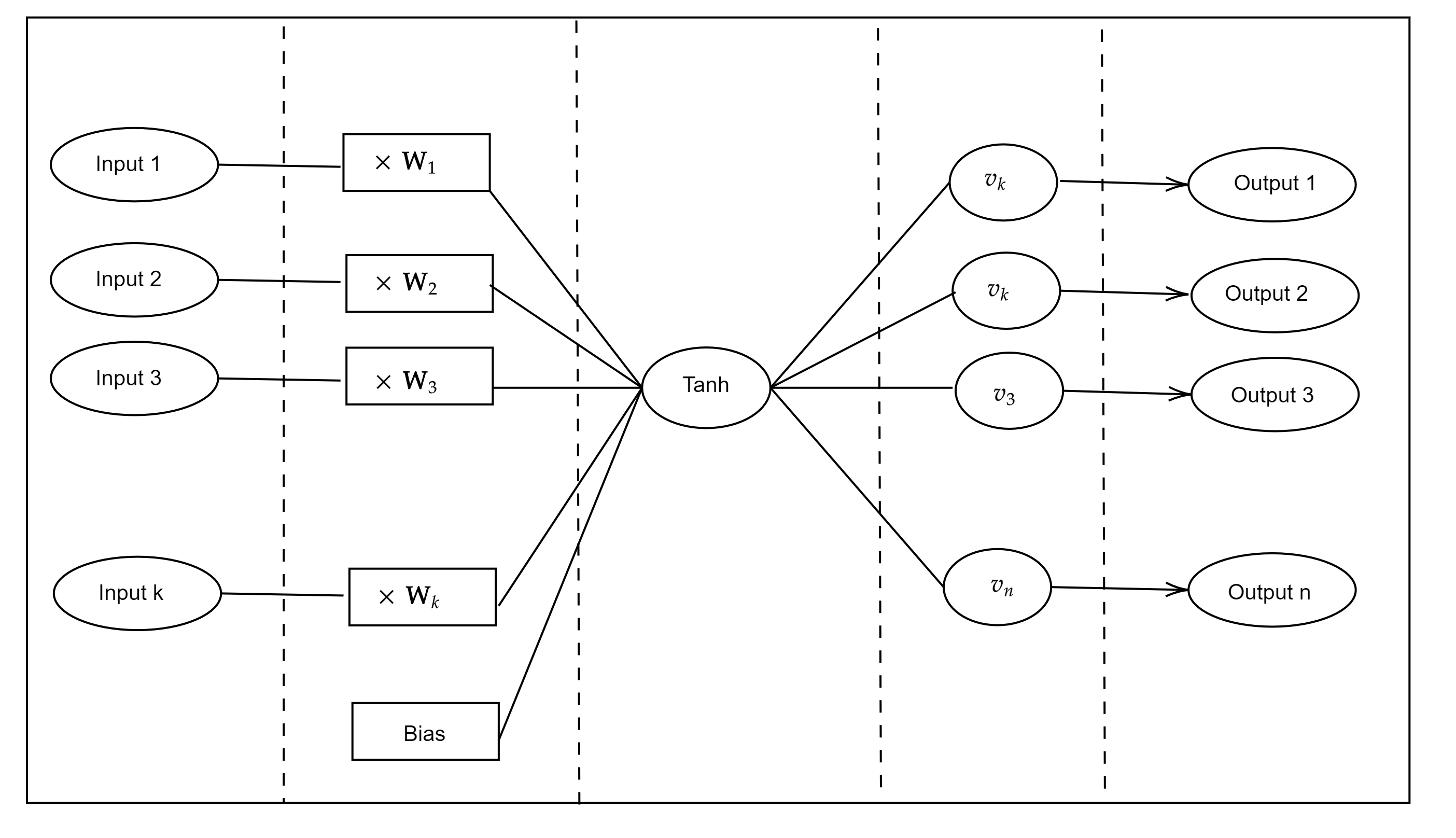}
		\caption{Legendre Neural Block.}
		\label{fig: Legendre block}
	\end{figure}

	\begin{figure}[!ht]
		\centering 
		\includegraphics[width=\linewidth]{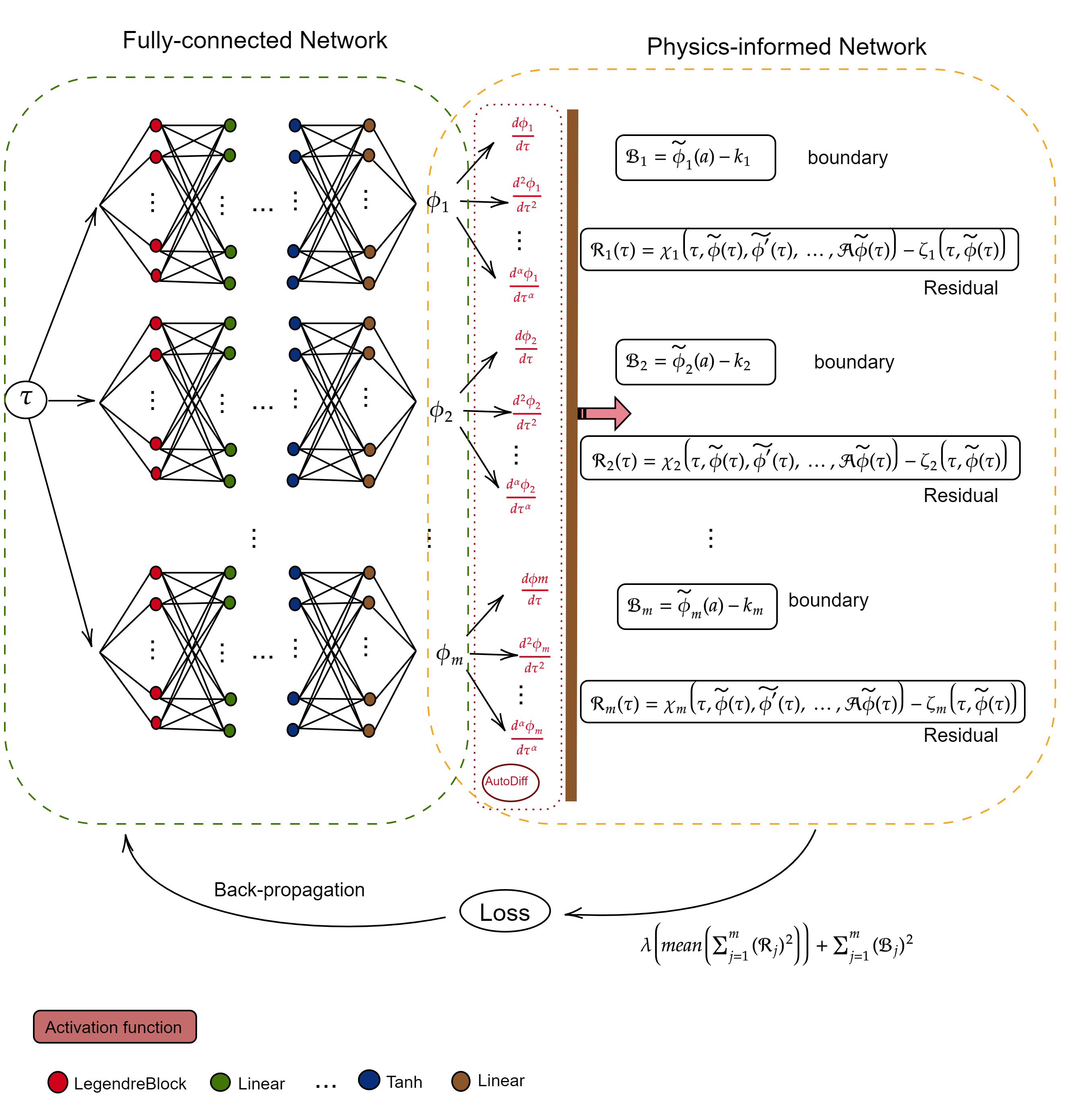}
		\caption{physics-informed neural networks.}
		\label{fig: pinn}
	\end{figure}
 
In Fig. \ref{fig: pinn}, we depict the architecture of the suggested physics-informed neural network. This network consists of residual terms derived from differential equations and initial conditions. The network processes its inputs to generate corresponding outputs which is the dynamics of the problem. The subsequent component is the physics-informed network, where the output $\phi$ is utilized to calculate derivatives based on provided equations, alongside the evaluation of boundary and initial conditions. The ultimate phase involves a feedback mechanism, wherein these residuals are minimized through optimization with respect to a learning rate, resulting in the adjustment of the neural network's parameters. To enhance model performance by parameter adjustments and minimizing the loss function, two common optimization techniques are employed: Adam and L-BFGS. Adam is a well-known first-order optimizer with an adaptive learning rate, while the latter is a second-order algorithm.

 \begin{theorem}
     The Limited-memory Broyden–Fletcher–Goldfarb–Shanno (L-BFGS) optimization algorithm updates the current value $\mathbf x_k$
  using the formula:
  \[
\mathbf{x}_{k+1} = \mathbf x_k - \alpha_k H_k \nabla f(\mathbf x_k),
\]
where:
\begin{itemize}
    \item $\alpha_k$  is the step size,
    \item $H_k$ is the approximation to the inverse Hessian matrix, and
    \item $\nabla f(\mathbf x_k)$ is the gradient of the objective function at $\mathbf x_k$.
\end{itemize}

The step size $\alpha_k$ is typically determined through line search methods.
The L-BFGS algorithm maintains a limited memory of past gradients and parameter changes to approximate the inverse Hessian matrix efficiently.
The update formula for $H_k$  involves vectors $s_k$ and $y_k$ representing the change in parameters and gradient, respectively:
\[
H_{k+1} = (I - \rho_k s_k y_k^T)H_k(I - \rho_k y_k s_k^T) + \rho_k s_k s_k^T,
\]
where $I$ is the identity matrix and $\rho_k=\frac{1}{y^{T}_ks_k}$.
 \end{theorem}

	\section{Numerical Examples} \label{sec4}
In order to demonstrate the efficacy of our proposed method, this section presents a comprehensive numerical solution to a diverse set of well-known delay differential equations and differential-algebraic equations. This exemplary set includes linear and nonlinear, equations involving fractional derivatives, pantograph DDEs, fractional pantograph delay differential equations, and fractional differential-algebraic equations. To gauge the precision of our calculations, we employ a suite of error metrics, including Mean Absolute Error $(MAE)$, $L_{1}$ norm, $L_{2}$ norm, and $L_{\infty}$ norm, each defined as illustrated in Tab. \ref{tab: metrics}. Here, $y_i$ represents the exact solution, and $\tilde{y_i}$ signifies the predicted output generated by our algorithm. In addition to presenting the numerical results and error metrics, we provide graphical representations of the exact and computed solutions. The experiments were conducted using Python programming language, Google Colab for computational resources, PyTorch framework for neural network development, and employing a learning rate of $0.01$ with both Adam and L-BFGS optimizers.

 \begin{table}[t]	
	\centering	
	\caption{Error formula}	
	\begin{tabular}{p{2cm} p{6cm}}
		\toprule
		Metrics & Formula \\
		\midrule
		$L_1$ Norm & $\sum_{i=1}^{n}\left |\hat{y}_{i}- y_{i} \right |$\\
		$L_2$ Norm &$\sqrt{\sum_{i=1}^{n}\left |\hat{y}_{i}- y_{i} \right |^{2}}$\\
		$L_{\infty}$ Norm & $\max\left |\hat{y}_{i}- y_{i} \right |,  \;\;\;\; i=1,...,n$\\
		Relative $\; L_2$ & $\frac{\left \| y_{i}-\hat{y}_{i} \right \|_{2}}{\left \| y_{i} \right \|_2}$\\
		MAE & $\frac{1}{n}\sum_{i=1}^{n}\left |\hat{y}_{i}- y_{i} \right |$\\
		\bottomrule
	\end{tabular} 		
	\label{tab: metrics}
\end{table}

 \begin{example}\label{ex1_4}
    We consider the following nonlinear delay differential equations:
\end{example}

\begin{equation}
    y^{\prime}(x)+e^xy^{\prime}(y-\sin(x^2))+\cos(x)y(x-\sin(x))=-e^{-x}-e^{\sin(x)^2}+\cos(x)e^{\sin(x)-x},\quad x\in[0,1]
\end{equation}
with initial condition
$y(0)=1$ and
the exact solution
$y(x)=e^{-x}$. Tab. \ref{tab_ex_1} presents the computed error values for this equation, and in Fig. \ref{fig1_4}, the difference between the actual equation solution and the predicted value, as well as the residual value, have been plotted. The architecture comprises six sequentially arranged trainable layers. It begins with the first layer as a Legendre Block with 16 nodes, followed by three fully connected layers with $\tanh(x)$ activations: the second layer with 32 nodes, the third layer with 64 nodes, and the fourth layer with 32 nodes. The final layer is another Legendre Block with 5 nodes.
 \begin{table}[ht] 
\centering
 \caption{Approximate Solutions and Error Values for Ex. \ref{ex1_4}}

\begin{tabular}{ll}
\toprule
                Name &   Values \\
\midrule
   
              $L_1$ Norm & $1.88e-02$ \\
              $L_2$ Norm & $1.25e-03$ \\
           $L_{\infty}$ Norm & $1.43e-04$ \\
        Relative $L_2$ & $1.09e-04$ \\
  Mean Absolute Error & $6.27e-05$ \\
\bottomrule
\end{tabular}
\quad
\begin{tabular}{ll}
\toprule
                  x &                  y \\
\midrule
 $0.0000000000000000$ & $0.9999691812777030$ \\
 $0.1000000000000000$ & $0.9048038580658174$ \\
 $0.2000000000000000$ & $0.8187027715677497$ \\
 $0.5000000000000000$ & $0.6064693827773249$ \\
 $1.0000000000000000$ & $0.3677365537432892$ \\
\bottomrule
\end{tabular}
\label{tab_ex_1}
\end{table}
\begin{figure}[ht]
		\centering
		\begin{subfigure}{.40\textwidth}
			\centering			\includegraphics[width=1\linewidth]{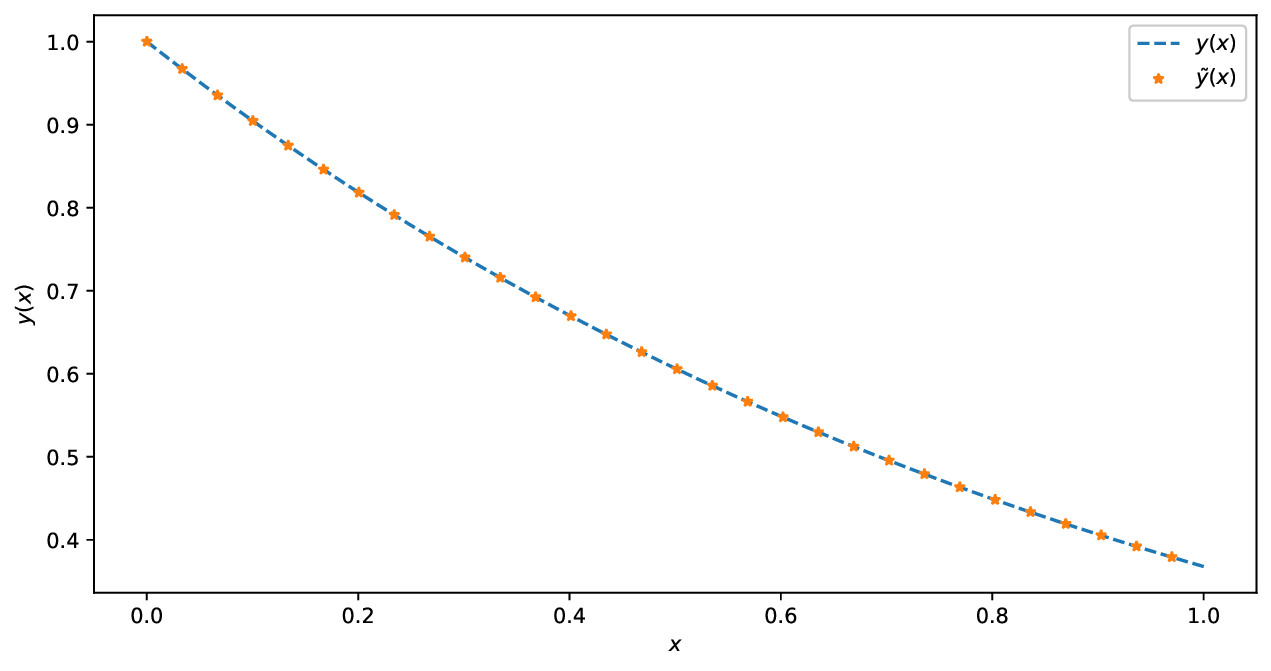}

		\end{subfigure}
		\begin{subfigure}{.40\textwidth}
			\centering
			\includegraphics[width=1\linewidth]{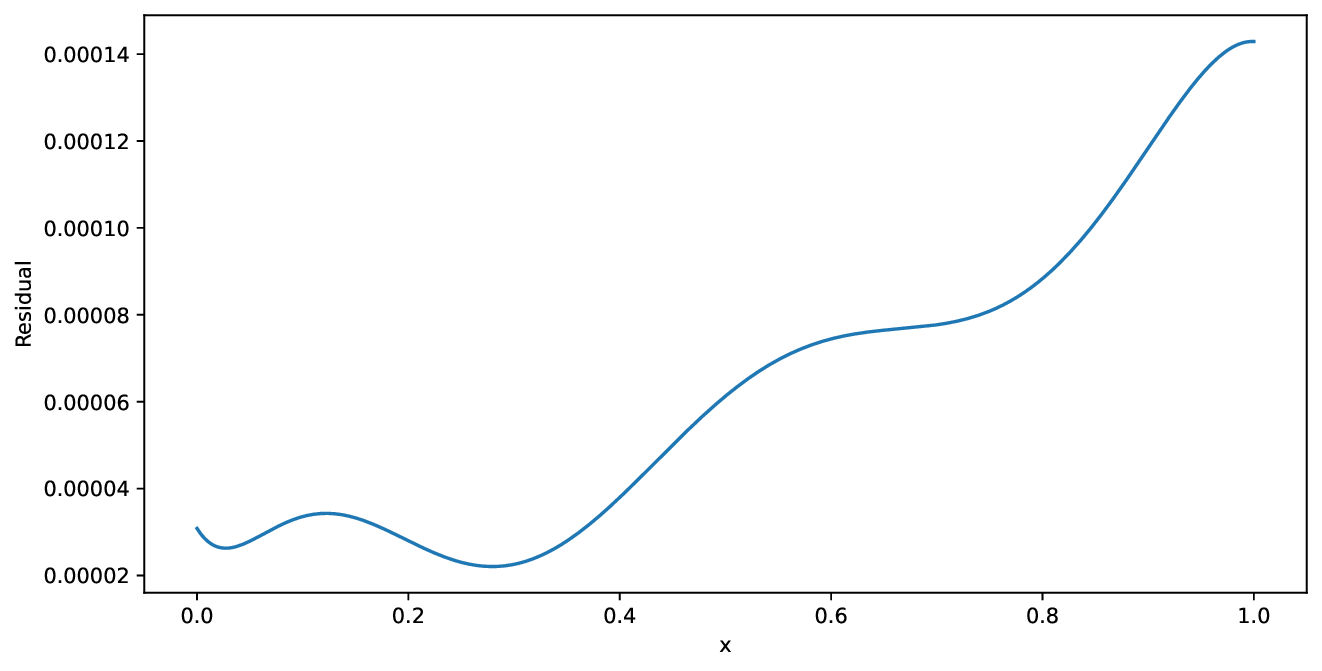}

		\end{subfigure}
   
		\caption{(a) Comparison of Results (b) Absolute Error of Ex.\ref{ex1_4}.}\label{fig1_4}
        \end{figure}
\begin{example} \label{ex2_4}
Consider the following nonlinear delay differential equation:

\begin{equation} 
\begin{aligned}
y'(x) + \sqrt{\cos(x)} y'(\sqrt{x}) + (\sin(\sqrt{x}) + e^x) y(\sin(x)) = e^x + \sqrt{\cos(x)} e^{\sqrt{x}} + (\sin(\sqrt{x}) + e^x) e^{\sin(x)}, \quad x \in [0, 1]
\end{aligned}
\end{equation}
with an initial condition of $y(0) = 1$ and the exact solution being $y(x) = e^x.$ In Tab. \ref{tab_ex_2}, approximate values and error values have been obtained. Furthermore, in Fig. \ref{fig2_4}, it is evident that the residual error values have been depicted with precision to five decimal places. The architecture is identical to that in example \ref{ex1_4}.
\end{example}
\begin{table}[t] 
\centering
 \caption{Approximate Solutions and Error Values for Ex. \ref{ex2_4}}

\begin{tabular}{ll}
\toprule
                Name &   Values \\
\midrule
   
              $L_1$ Norm & $2.96e-02$ \\
              $L_2$ Norm & $1.94e-03$ \\
           $L_{\infty}$ Norm & $1.84e-04$ \\
        Relative $L_2$ & $6.26e-05$ \\
  Mean Absolute Error & $9.87e-05$ \\
\bottomrule
\end{tabular}
\quad
\begin{tabular}{ll}
\toprule
                  x &                  y \\
\midrule
 $0.0000000000000000$ & $1.0001576081160477$ \\
 $0.1000000000000000$ & $1.1053201283380465$ \\
 $0.2000000000000000$ & $ 1.2214592190076856$ \\
 $0.5000000000000000$ & $1.6489053832574243$ \\
 $1.0000000000000000$ & $2.7183227015316742$ \\
\bottomrule
\end{tabular}
\label{tab_ex_2}
\end{table}
\begin{figure}[t]
		\centering
		\begin{subfigure}{.40\textwidth}
			\centering			\includegraphics[width=1\linewidth]{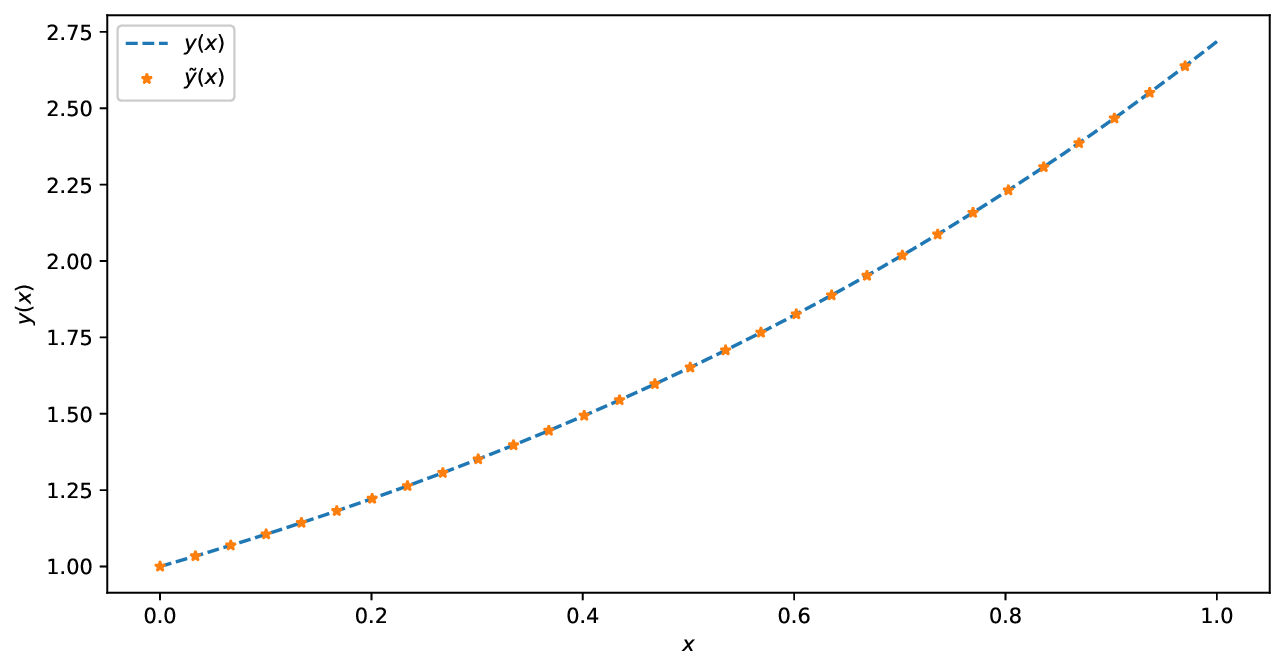}

		\end{subfigure}
		\begin{subfigure}{.40\textwidth}
			\centering
			\includegraphics[width=1\linewidth]{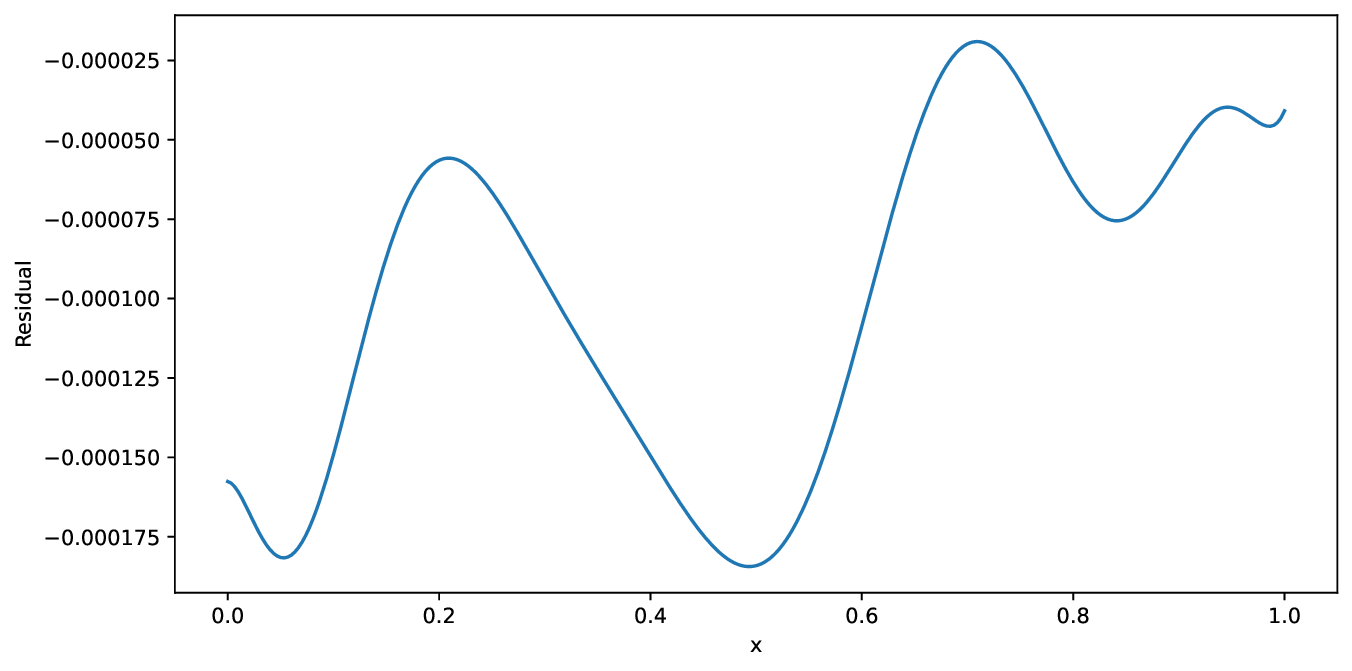}

		\end{subfigure}
   
		\caption{(a) Comparison of Results (b) Absolute Error of Ex.\ref{ex2_4}.}\label{fig2_4}
        \end{figure}

    \begin{example}\label{ex3_4}
We have the following delay differential equation, known as a Pantograph equation:
\begin{equation}
y^{\prime}(x)=\frac{1}{2}y(qx)-y(x)-\frac{1}{2}e^{-qx}, \quad x\in[0,1], y(0)=1.
\end{equation}

The exact solution to this Pantograph equation is given by
$y(x)=e^{-x}$ for $0<q<1$. In Figure \ref{fig3_4}, the exact and estimated solutions are plotted for $q=0.5$. The architecture consists of six trainable layers arranged sequentially, commencing with the Legendre Block as the initial layer. Subsequently, fully connected layers are incorporated, employing $\tanh(x)$ activation functions. The node configuration for each layer is specified as follows: 16 nodes in the Legendre Block layer, 32 nodes in both the second and fourth layers, 64 nodes in the third layer, and 5 nodes in the final layer. (see Tab. \ref{tab_ex_3})
\begin{table}[ht] 
\centering
 \caption{Approximate Solutions and Error Values for Ex. \ref{ex3_4}}

\begin{tabular}{ll}
\toprule
                Name &   Values \\
\midrule
   
              $L_1$ Norm & $7.16e-04$ \\
              $L_2$ Norm & $4.75e-05$ \\
           $L_{\infty}$ Norm & $4.59e-06$ \\
        Relative $L_2$ & $4.17e-06$ \\
  Mean Absolute Error & $2.39e-06$ \\
\bottomrule
\end{tabular}
\quad
\begin{tabular}{ll}
\toprule
                  x &                  y \\
\midrule
 $0.0000000000000000$ & $1.0000022668994222$ \\
 $0.1000000000000000$ & $0.9048372111735443$ \\
 $0.2000000000000000$ & $0.8187349625675631$ \\
 $0.5000000000000000$ & $0.6065352015875470$ \\
 $1.0000000000000000$ & $0.3678792914460537$ \\
\bottomrule
\end{tabular}
\label{tab_ex_3}
\end{table}
\begin{figure}[ht]
		\centering
		\begin{subfigure}{.40\textwidth}
			\centering			\includegraphics[width=1\linewidth]{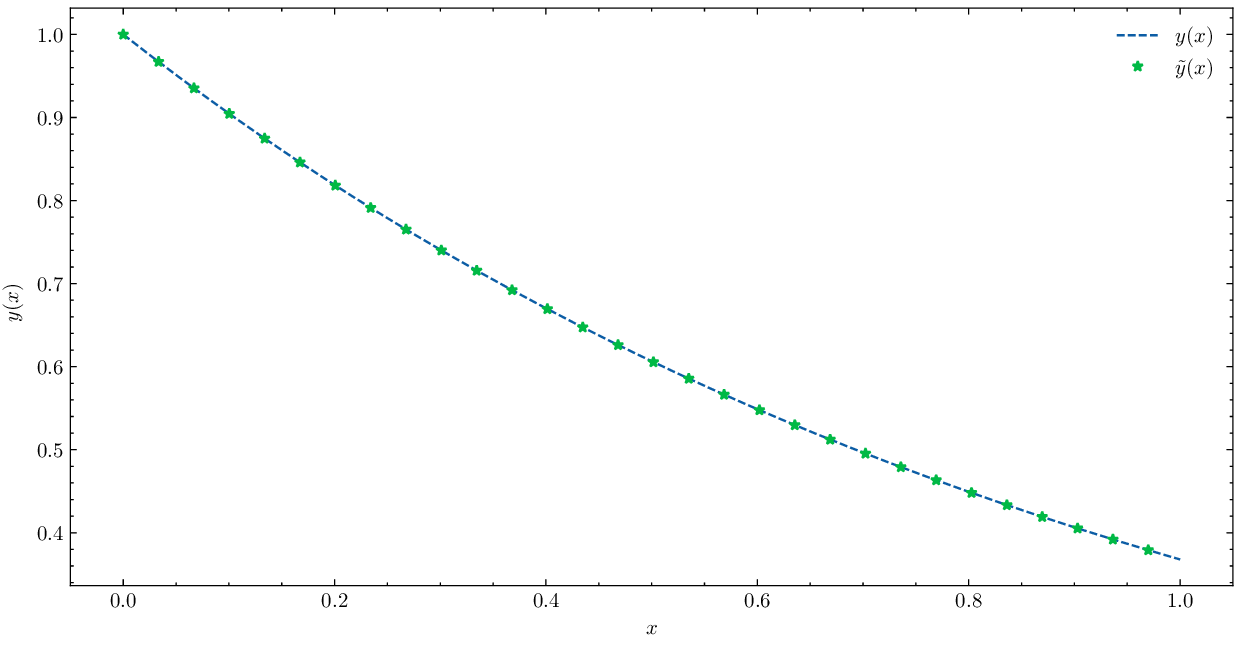}

		\end{subfigure}
		\begin{subfigure}{.40\textwidth}
			\centering
			\includegraphics[width=1\linewidth]{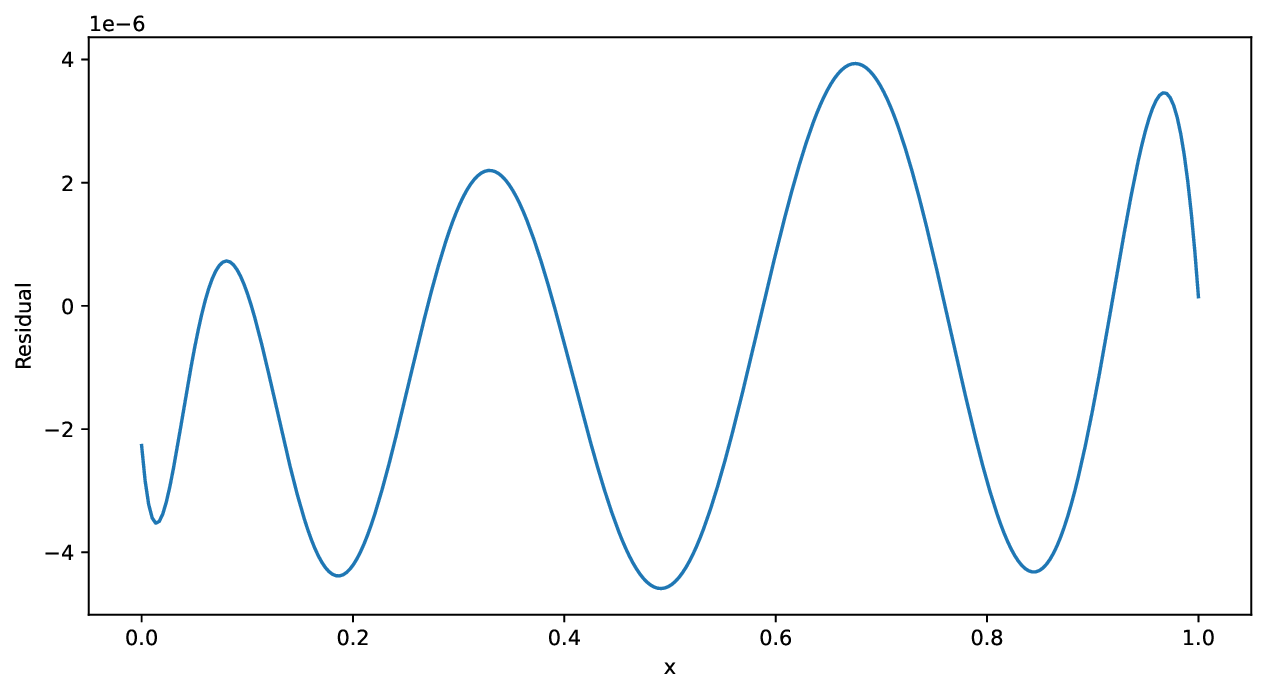}

		\end{subfigure}
   
		\caption{(a) Comparison of Results (b) Absolute Error of Ex.\ref{ex3_4}.}\label{fig3_4}
        \end{figure}

\end{example}
    \begin{example}\label{ex4_4}
    Consider the following delay differential equation, which is a Pantograph equation:
    \begin{equation}
        y^{\prime}(x) = \frac{1}{2}y(x) + \frac{1}{2}e^{\frac{x}{2}}y\left(\frac{x}{2}\right), \quad 0 \leq x \leq R, \quad y(0) = 1.
    \end{equation}
    
    The exact solution to this Pantograph equation is $y(x) = e^x$, where we take $R = 1$. Fig. \ref{fig4_4} and Tab. \ref{tab_ex_4} present a comparative analysis of results, showcasing absolute errors and approximate solutions. The architecture is the same as the one in example \ref{ex3_4}.
\end{example}
\begin{table}[t] 
\centering
 \caption{Approximate Solutions and Error Values for Ex. \ref{ex4_4}}
\begin{tabular}{ll}
\toprule
                Name &   Values \\
\midrule
   
              $L_1$ Norm & $6.51e-03 $ \\
              $L_2$ Norm & $4.12e-04$ \\
           $L_{\infty}$ Norm & $4.07e-05$ \\
        Relative $L_2$ & $ 1.33e-05$ \\
  Mean Absolute Error & $2.17e-05$ \\
\bottomrule
\end{tabular}
\quad
\begin{tabular}{ll}
\toprule
                  x &                  y \\
\midrule
 $0.0000000000000000$ & $ 0.9999892555021601$ \\
 $0.1000000000000000$ & $ 1.1051505603136731$ \\
 $0.2000000000000000$ & $1.2213964499059211$ \\
 $0.5000000000000000$ & $1.6487125644152036$ \\
 $1.0000000000000000$ & $2.7182473537536378$ \\
\bottomrule
\end{tabular}
\label{tab_ex_4}
\end{table}
\begin{figure}[t]
		\centering
		\begin{subfigure}{.40\textwidth}
			\centering			\includegraphics[width=1\linewidth]{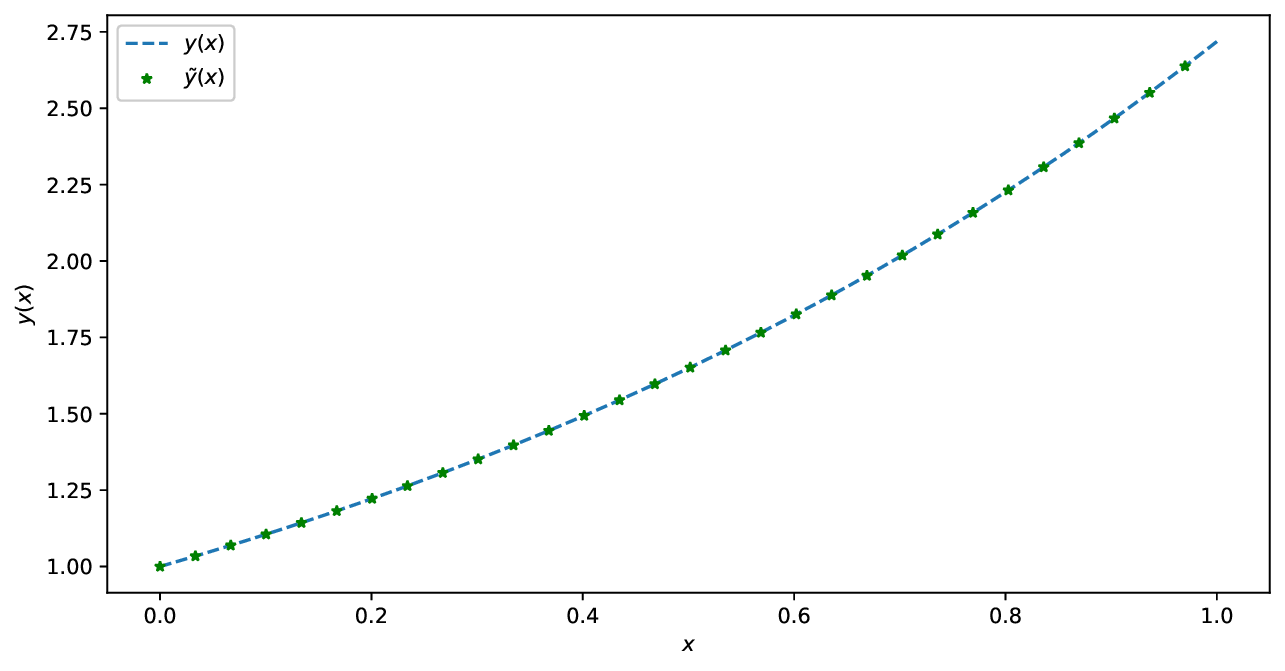}

		\end{subfigure}
		\begin{subfigure}{.40\textwidth}
			\centering
			\includegraphics[width=1\linewidth]{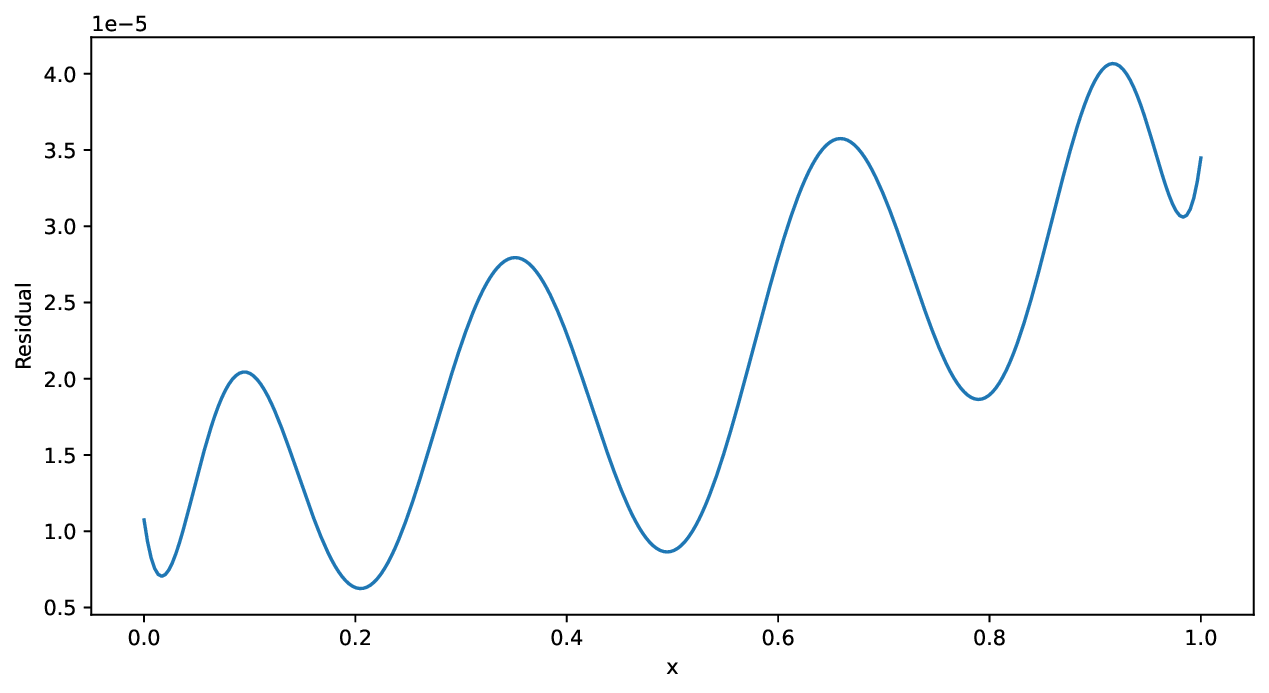}

		\end{subfigure}
   
		\caption{ (a) Comparison of Results (b) Absolute Error of Ex. \ref{ex4_4}.}\label{fig4_4}
        \end{figure}

\begin{example}\label{ex5_4}
    We have the following fractional delay differential equation:
    \begin{equation}
        D^{0.3}y(x) = y(x-1) - y(x) + 1 - 3x + 3x^2 + \frac{2000x^{2.7}}{1071\Gamma(0.7)}, \qquad y(0) = 0,
    \end{equation}
    where the analytical solution is $y(x) = x^3$ for $x \geq 0$. In Fig. \ref{fig5_4}, a comparison of outcomes is presented, illustrating absolute errors. Correspondingly, Tab. \ref{tab_ex_5} provides the details of approximate solutions and absolute errors.
\end{example}
\begin{table}[ht]
\centering
 \caption{Approximate Solutions and Error Values for Ex. \ref{ex5_4}}
\begin{tabular}{ll}
\toprule
                Name &   Values \\
\midrule
   
              $L_1$ Norm & $5.02e-01 $ \\
              $L_2$ Norm & $1.96e-02$ \\
           $L_{\infty}$ Norm & $ 1.09e-03$ \\
        Relative $L_2$ & $ 1.63e-03$ \\
  Mean Absolute Error & $5.02e-04$ \\
\bottomrule
\end{tabular}
\quad
\begin{tabular}{ll}
\toprule
                  x &                  y \\
\midrule
 $0.0000000000000000$ & $-0.0000532769776176$ \\
 $0.1000000000000000$ & $ 
  0.0014034548414760$ \\
 $0.2000000000000000$ & $0.0083881827921961$ \\
 $0.5000000000000000$ & $0.1237044632150650$ \\
 $1.0000000000000000$ & $0.9992348206216326$ \\
\bottomrule
\end{tabular}
\label{tab_ex_5}
\end{table}
\begin{figure}[ht]
		\centering
		\begin{subfigure}{.40\textwidth}
			\centering			\includegraphics[width=1\linewidth]{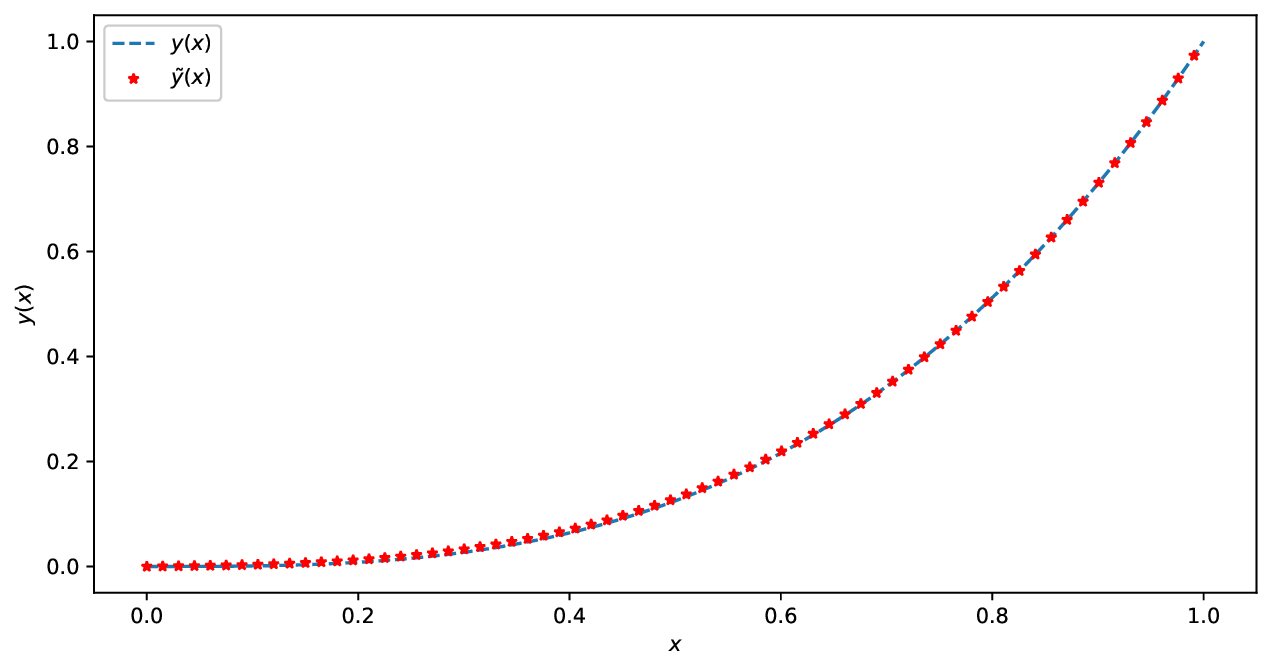}

		\end{subfigure}
		\begin{subfigure}{.40\textwidth}
			\centering
			\includegraphics[width=1\linewidth]{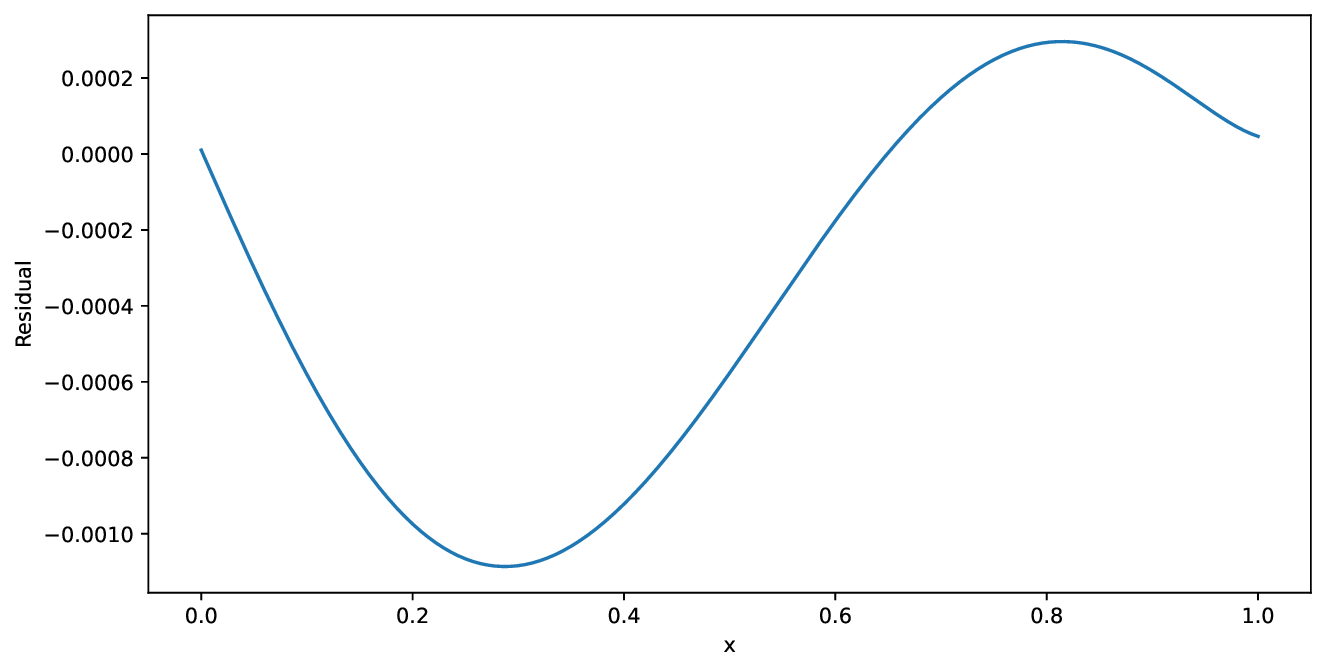}

		\end{subfigure}
   
		\caption{(a) Comparison of Results (b) Absolute Error of Ex. \ref{ex5_4}.}\label{fig5_4}
        \end{figure}

\begin{example}\label{ex6_4}
    Consider the following fractional-order delay differential equation:
    \begin{equation}
        D^{\frac{1}{2}}y(x) = y(\frac{1}{3}x) + y^2(x) + g(x), \quad y(0) = 0,
    \end{equation}
    where $g(x) = \frac{8}{3\sqrt{\pi}}x^{\frac{3}{2}} - \frac{2}{\sqrt{\pi}}x^{\frac{1}{2}} - x^4 + 2x^3 - \frac{10}{9}x^2 + \frac{1}{3}x$, and the equation has an analytical solution given by $y(x) = x^2 - x$. Fig. \ref{fig6_4} displays a comparative analysis of outcomes, highlighting absolute errors. Additionally, Tab. \ref{tab_ex_6} details the corresponding approximate solutions and absolute errors. The layer architecture in examples \ref{ex5_4} and \ref{ex6_4} is identical to that in example \ref{ex3_4}.
\end{example}
\begin{table}[b]
\centering
 \caption{Approximate Solutions and Error Values for Ex. \ref{ex6_4}}
\begin{tabular}{ll}
\toprule
                Name &   Values \\
\midrule
   
              $L_1$ Norm & $3.09e-02 $ \\
              $L_2$ Norm & $1.11e-03$ \\
           $L_{\infty}$ Norm & $ 7.50e-05$ \\
        Relative $L_2$ & $ 1.92e-04$ \\
  Mean Absolute Error & $3.09e-05$ \\
\bottomrule
\end{tabular}
\quad
\begin{tabular}{ll}
\toprule
                  x &                  y \\
\midrule
 $0.0000000000000000$ & $ 0.0000065602927927$ \\
 $0.1000000000000000$ & $ -0.0899926607745801$ \\
 $0.2000000000000000$ & $-0.1599556891727947$ \\
 $0.5000000000000000$ & $-0.2499623639493245$ \\
 $1.0000000000000000$ & $ -0.0000014611903909 $ \\
\bottomrule
\end{tabular}
\label{tab_ex_6}
\end{table}
\begin{figure}[t]
		\centering
		\begin{subfigure}{.40\textwidth}
			\centering			\includegraphics[width=1\linewidth]{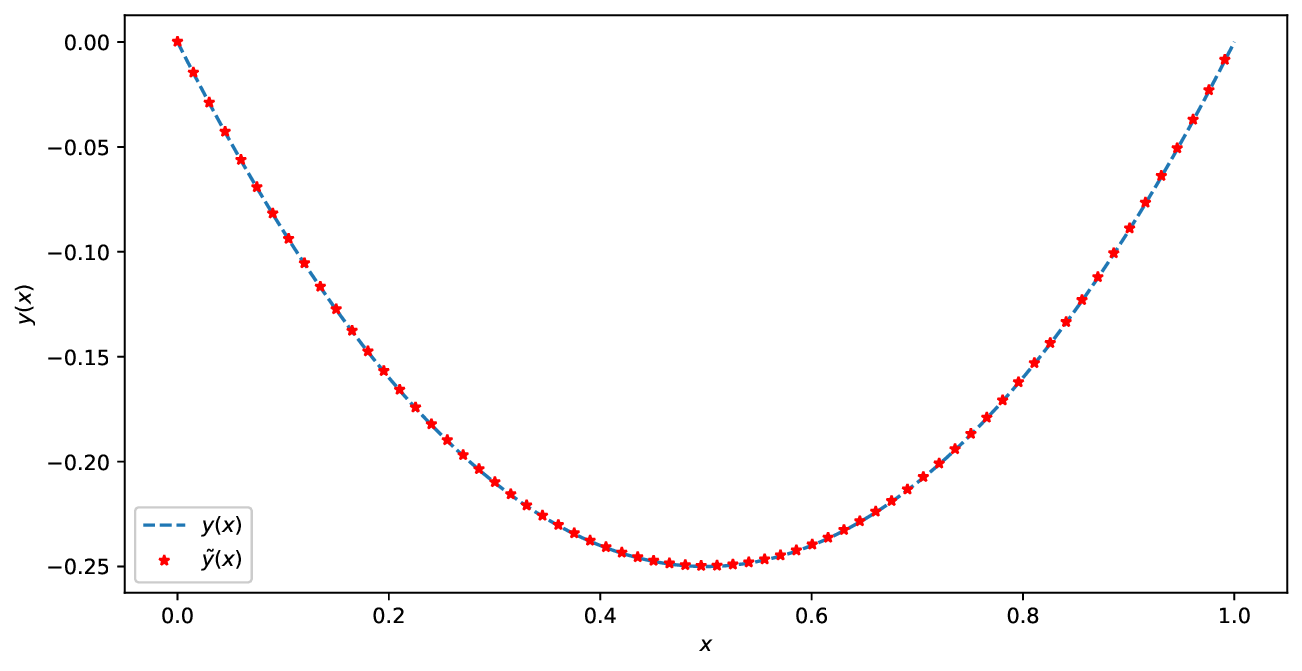}

		\end{subfigure}
		\begin{subfigure}{.40\textwidth}
			\centering
			\includegraphics[width=1\linewidth]{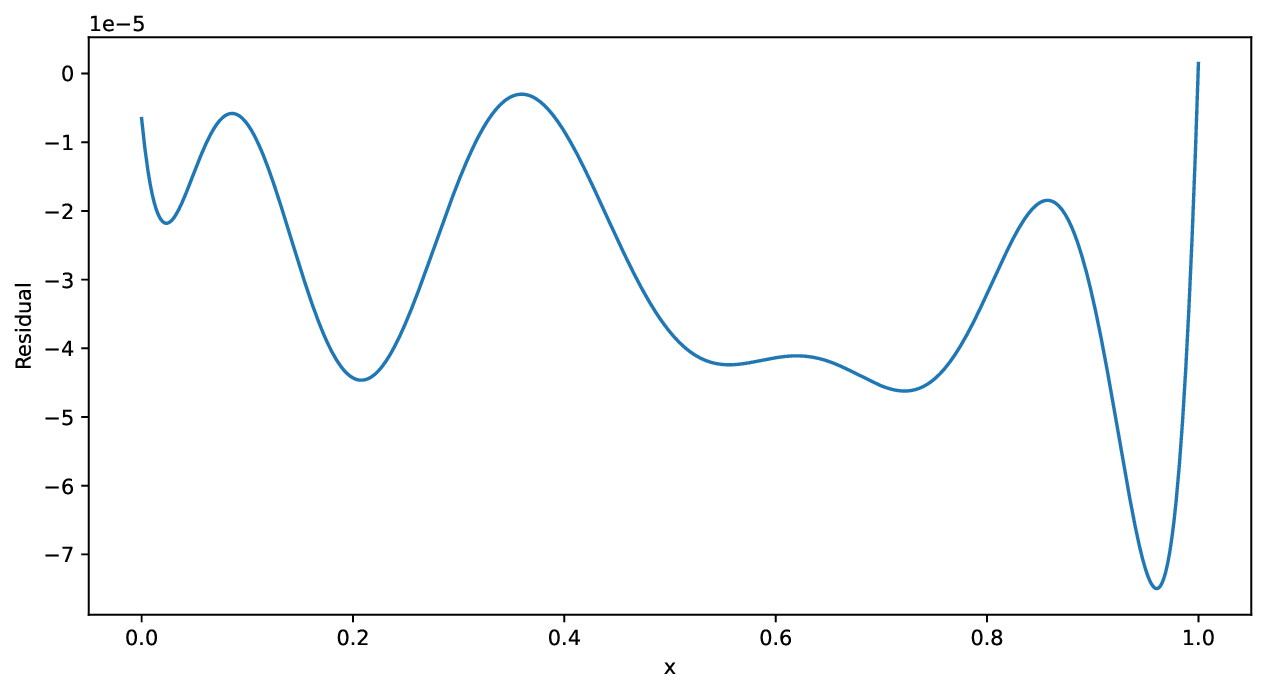}

		\end{subfigure}
   
		\caption{(a) Comparison of Results (b) Absolute Error of Ex. \ref{ex6_4}.}\label{fig6_4}
        \end{figure}

In the following examples, we address a differential-algebraic equation and a fractional-order differential-algebraic equation.        

\begin{example} \label{ex7_4}
consider the following DAE:
\begin{equation}
\begin{cases}

y^{\prime}_1(x)=y_1(x)-y_3(x)y_2(x)+sin(x)+xcos(x),
\\

y^{\prime}_2(x)=xy_3(x)+y^2_1(x)+sec^2(x)-x^2(cos(x)+sin^2(x)),
\\

0=y_1(x)-y_3(x)+x(cos(x)-sin(x)), \quad x\in [0,1],
\\
\end{cases}
\end{equation}
the exact solutions are given with the initial conditions $y_1(0) = y_2(0) = y_3(0) = 0$: $y_1(x) = x\sin(x)$, $y_2(x) = \tan(x)$, and $y_3(x) = x\cos(x)$. In Tab. \ref{tab7}, we present the error for this case. Additionally, in Fig. \ref{fig7_4}, we have plotted the absolute error and the comparison between exact and calculated solutions for $y_1(x)$, $y_2(x)$, and $y_3(x)$. Three multilayer perceptron models with identical layer architectures are utilized. Each model starts with a Legendre Block having 40 nodes, followed by a series of fully connected layers: a layer with 5 nodes and $\tanh(x)$ activation, a layer with 32 nodes and $\tanh(x)$ activation, another layer with 5 nodes and $\tanh(x)$ activation, followed by linear layers with 10 and output nodes.
\end{example}
\begin{table}[t]
\centering
 \caption{Error Values for Ex. \ref{ex7_4}}\label{tab7}
\begin{tabular}{llll}
\toprule
                Name &   $y_1$ Error Values & $y_2$ Error Values & $y_3$ Error Values \\
\midrule
   
              $L_1$ Norm & $7.45e-02 $ & $2.48e-02$ & $ 1.84e-01$ \\
              $L_2$ Norm & $2.65e-03$ & $ 1.27e-03$ & $8.01e-03$ \\
           $L_{\infty}$ Norm & $ 1.09e-04$ & $1.07e-04$ & $ 8.15e-04$ \\
        Relative $L_2$ & $ 2.37e-04$ & $6.03e-05$ & $6.75e-04$ \\
  Mean Absolute Error & $9.31e-05$ & $3.10e-05$ & $2.30e-04$ \\
\bottomrule
\end{tabular}
\end{table}
\begin{figure}[ht]
		\centering
 \includegraphics[width=1\linewidth]{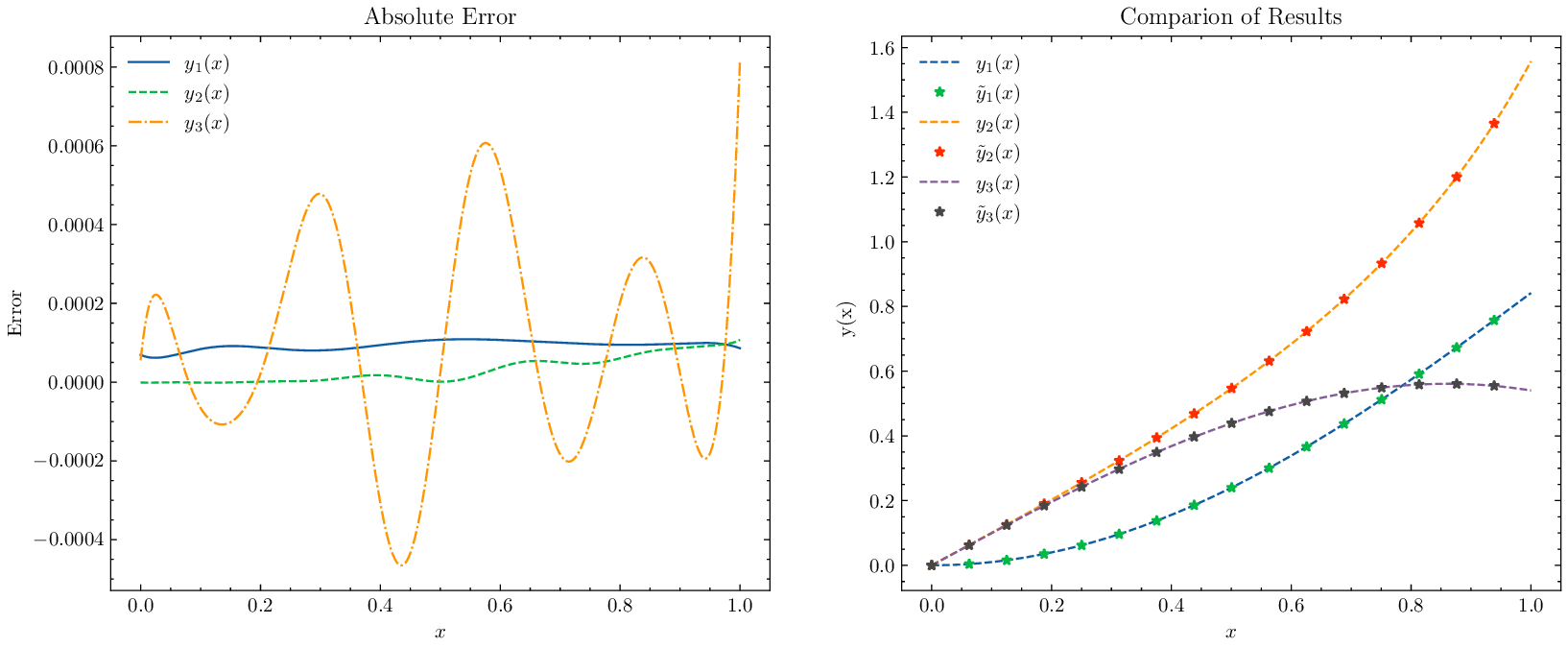}

		\caption{Absolute error and comparison of results for Ex. \ref{ex7_4}.}\label{fig7_4}
        \end{figure}

\begin{example} \label{ex8_4}
    In this example, we compute the solution for a linear fractional differential-algebraic equation

    \begin{equation}
\begin{cases}

D^{\frac{1}{2}}y_1(x)+2y_1(x)-\frac{\Gamma(\frac{7}{2})}{\Gamma(3)}y_2(x)+y_3(x)=2x^{\frac{5}{2}}+sin(x),
\\

D^{\frac{1}{2}}y_2(x)+y_2(x)+y_3(x)=\frac{\Gamma(3)}{\Gamma(\frac{5}{2})}x^{\frac{3}{2}}+x^2+sin(x),
\\

2y_1(x)+y_2(x)-y_3(x)=2x^{\frac{5}{2}}+x^2-sin(x), \quad x\in[0,1]
\\
\end{cases}
\end{equation}
with zero initial conditions, the exact solutions are as follows: $y_1(x) = x^{5/2}$, $y_2(x) = x^2$, and $y_3(x) = sin(x)$. Within Tab. \ref{tab8}, we showcase the error corresponding to this scenario. Furthermore, Fig. \ref{fig8_4} illustrates both the absolute error and the contrast between the exact and computed solutions for $y_1(x)$, $y_2(x)$, and $y_3(x)$. Three PyTorch double-precision multilayer perceptron models exhibit a consistent layer configuration: Legendre Block (40 nodes), fully connected layers (5 nodes with $\tanh(x)$, 32 nodes with $\tanh(x)$, 5 nodes with $\tanh(x)$), Chebyshev Block (10 nodes), and an output layer.
\end{example}

\begin{table}[ht] 
\centering
 \caption{Error Values for Ex. \ref{ex8_4}}\label{tab8}
\begin{tabular}{llll}
\toprule
                Name &   $y_1$ Error Values & $y_2$ Error Values & $y_3$ Error Values \\
\midrule
   
              $L_1$ Norm & $2.16e-02$ & $1.57e-01$ & $ 1.33e-01$ \\
              $L_2$ Norm & $8.85e-04$ & $5.32e-03$ & $4.92e-03$ \\
           $L_{\infty}$ Norm & $ 1.47e-04$ & $4.00e-04$ & $3.17e-04$ \\
        Relative $L_2$ & $ 6.85e-05$ & $3.76e-04$ & $2.98e-04$ \\
  Mean Absolute Error & $2.16e-05$ & $ 1.57e-04$ & $1.33e-04$ \\
\bottomrule
\end{tabular}
\end{table}
\begin{figure}[ht] 
		\centering
 \includegraphics[width=1\linewidth]{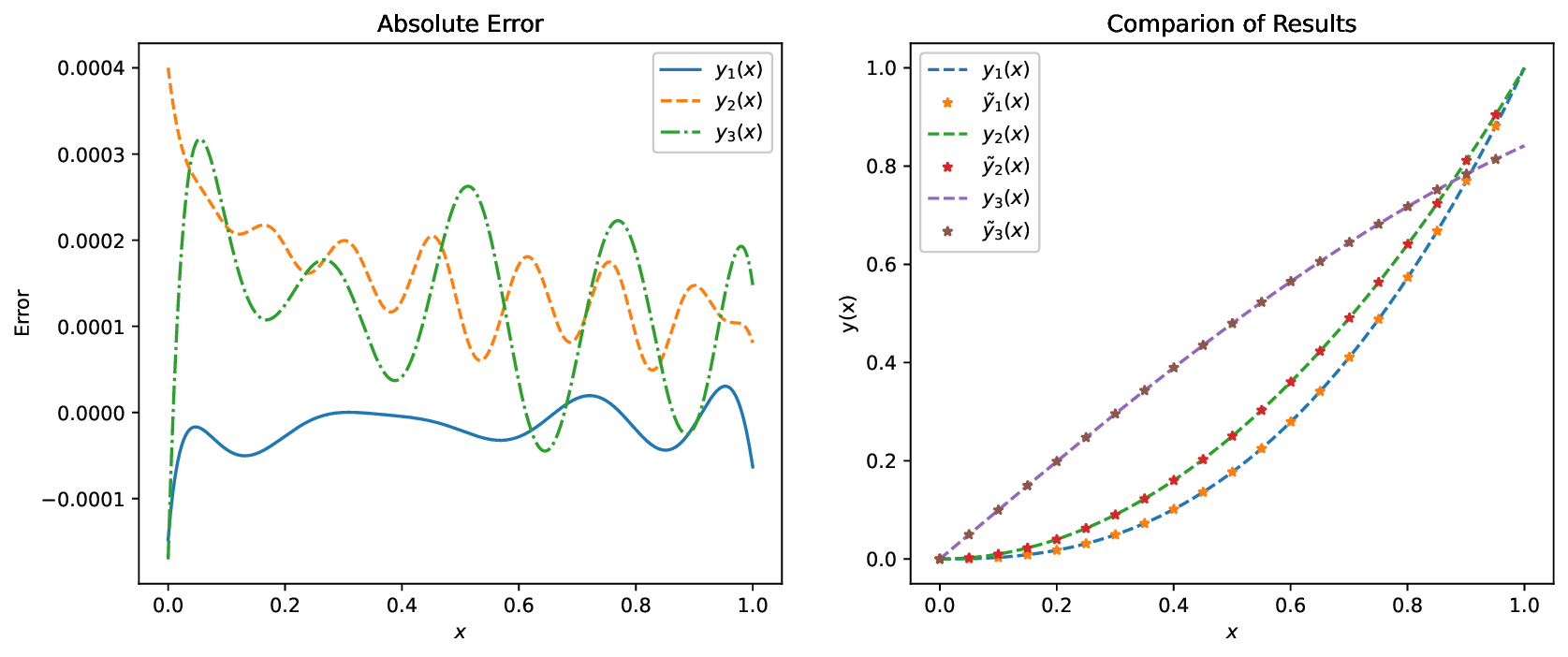}

		\caption{Absolute error and comparison of results for Ex. \ref{ex8_4}.}\label{fig8_4}
        \end{figure}
        
        \begin{figure}[!ht]
		\centering 
		\includegraphics[width=.6\linewidth]{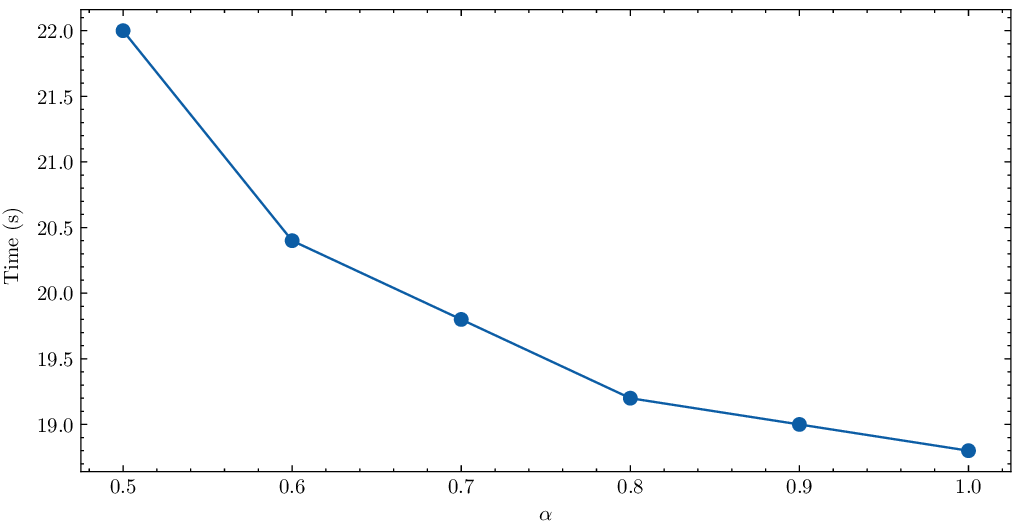}
		\caption{Time complexity of training fPINN using the operational matrix of Caputo derivative using 1000 Adam epochs.}
		\label{fig: time}
	\end{figure} 

 \section{Conclusion}\label{sec5} 

While automatic differentiation is employed for derivative operations of integer order, we proposed a non-uniform discretization of the fractional Caputo operator to handle differential equations of fractional order.
  
In our proposed framework, distinguished by non-uniform discretization, we observe that it does not impact the computational speed in the computation of fractional orders. When compared to automatic differentiation for integral orders, there is no significant decrease in speed as shown in Fig. \ref{fig: time}. This is because the operational matrix $\mathcal{A}$ can be computed before the training process. In addition, computing fractional derivatives in the training process can be done by matrix-vector multiplication.

Successfully applied to Delay Differential Equations, pantograph Delay Differential Equations, and Differential-algebraic Equations characterized by fractional orders, the method demonstrates efficacy in accurately solving these equations.

Implemented on the Legendre Neural Block architecture, which integrates Legendre polynomials into the artificial neural network structure, our approach capitalizes on properties such as nonlinearity, computational efficiency, and the ability to address vanishing/exploding gradient issues, ensuring correct backpropagation through continuous differentiability. Additionally, matrix operations streamline the computation of derivatives for Legendre polynomials in terms of themselves.

However, the fPINN exhibits limitations, notably the absence of guaranteed convergence in optimization error and relatively lower speed and accuracy compared to other machine learning methods, such as LS-SVM. Future endeavors may extend this method to address diverse types of fractional differential equations, encompassing partial and integral differential equations. Furthermore, exploring the approach for fractional orders greater than two holds promise for advancing research and applications.

\printbibliography
\end{document}